\documentclass[runningheads]{llncs}
\usepackage[T1]{fontenc}
\usepackage{graphicx}
\usepackage{orcidlink}
\usepackage{multirow}
\usepackage{booktabs}
\usepackage{amsmath, amssymb}
\usepackage{xcolor}
\usepackage{marvosym}

\newcommand{\pos}[1]{\textcolor{green!50!black}{\textbf{#1}}}

\begin{document}
\title{LGCA: Enhancing Semantic Representation via Progressive Expansion}
%
%
\author{Cao Thanh Hieu\inst{1,3}\orcidlink{0009-0002-3618-529X} \and
Tran Trung Khang\inst{2,3}\orcidlink{0009-0007-4753-0948} \and
Pham Gia Thinh\inst{1,3}\orcidlink{0009-0009-2465-6833}\and
Diep Tuong Nghiem\inst{1,3}\orcidlink{0000-0001-7406-1250}\and
Nguyen Thanh Binh\inst{1,3}\textsuperscript{(\Letter)}\orcidlink{0000-0001-5249-9702}}
\authorrunning{T. H. Cao et al.}
%
\institute{University of Science, Vietnam National University Ho Chi Minh City, Vietnam \and
National University of Singapore, Singapore \and 
AISIA Lab, Ho Chi Minh City, Vietnam
}
\maketitle              
\begin{abstract}
Recent advancements in large-scale pretraining in natural language processing have enabled pretrained vision-language models such as CLIP to effectively align images and text, significantly improving performance in zero-shot image classification tasks. Subsequent studies have further demonstrated that cropping images into smaller regions and using large language models to generate multiple descriptions for each caption can further enhance model performance. However, due to the inherent sensitivity of CLIP, random image crops can introduce misinformation and bias, as many images share similar features at small scales. To address this issue, we propose Localized-Globalized Cross-Alignment (LGCA), a framework that first captures the local features of an image and then repeatedly selects the most salient regions and expands them. The similarity score is designed to incorporate both the original and expanded images, enabling the model to capture both local and global features while minimizing misinformation. Additionally, we provide a theoretical analysis demonstrating that the time complexity of LGCA remains the same as that of the original model prior to the repeated expansion process, highlighting its efficiency and scalability. Extensive experiments demonstrate that our method substantially improves zero-shot performance across diverse datasets, outperforming state-of-the-art baselines.

\end{abstract}

\section{Introduction}\label{sec: Intro}
\begin{figure}[t]
    \centering
    \includegraphics[width=0.95\linewidth]{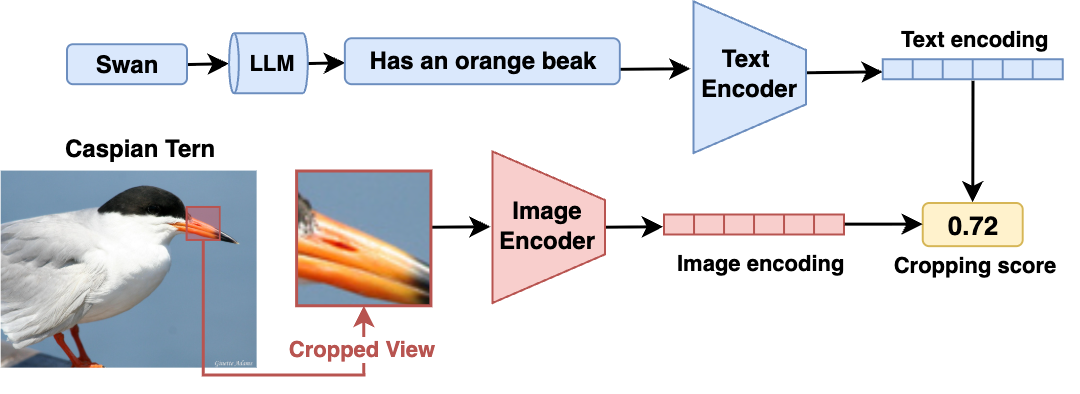}
    \caption{An illustrative case when random cropping introduces misleading similarity. Consider an image of a Caspian Tern paired with a caption of a swan. The LLM-generated description for the swan includes the phrase “has an orange beak.” Due to random cropping, the model captures only the beak region of the Caspian Tern, which also appears orange. This results in a high similarity score of 0.72, thereby distorting the overall similarity assessment.}
    \label{fig:no_exp} 
\end{figure}
Zero-shot classification between images and text seeks to align visual content with natural language in a shared latent space. This task has gained momentum thanks to large-scale pretraining in NLP \cite{devlin1810pre,radford2018improving,radford2019language,mann2020language}, enabling vision-language models (VLMs) such as CLIP \cite{radford2021learning} to achieve strong cross-modal understanding. However, CLIP’s performance is highly sensitive to prompt phrasing at inference \cite{radford2021learning,zhou2022learning}. For example, \cite{zhou2022learning} showed that altering “a photo of [CLASS]” to “a photo of a [CLASS]” improved accuracy by up to 6\%. This reliance on prompt engineering, often domain-specific and time-consuming, limits the model’s practical deployment \cite{zhou2022learning}.\\
To address this issue, \cite{menon2022visual} and \cite{pratt2023does} proposed leveraging large language models (LLMs) to automatically generate multiple refined text descriptions for each category. This approach alleviates the need for extensive manual prompt engineering, while also eliminating the requirement for additional fine-tuning. As a result, it enables models to maintain their generalization capabilities, which is particularly important in the context of prompt-learning methods. Research by \cite{li2022ordinalclip,wang2022learning,wu2023pi,tanwisuth2023pouf} highlights that these methods are prone to overfitting on the training data, making generalization a critical challenge.\\
LLM-based visual-text alignment typically emphasizes global matching, aligning text with the entire image, which is not always optimal. \cite{li2024visual} showed that fine-grained descriptions often map better to specific regions (e.g., “a woodpecker has a straight and pointed bill”) than to the whole image, but such localized focus can harm overall performance by ignoring broader context. To address this, they proposed the Weighted Visual-Text Cross-Alignment framework: images are divided into localized regions via random cropping, each region is weighted by its similarity to the full image, and then cross-aligned with caption descriptions (see Figures \ref{fig:img_crop} and \ref{fig:text_crop}; details in Section \ref{sec: Prem}). A key limitation of this approach is its sensitivity to the choice of cropped regions, much like CLIP’s sensitivity to prompt phrasing. For example, if a caption describes a swan’s orange beak but a crop instead captures the beak of a Caspian tern, the model may assign a high similarity score due to overlapping features. Such cases risk misinformation by incorrectly aligning image-text pairs, as illustrated in Figure \ref{fig:no_exp}.\\
To address this challenge, we propose \textbf{L}ocalized-\textbf{G}eneralized \textbf{C}ross-\textbf{A}lignment (LGCA). This method resolves the aforementioned issue by initially focusing on local crops of the image to capture fine-grained details. It then identifies the most salient local regions based on similarity scores, expands the image in both directions, and feeds it back into the model. This expansion process repeats, each time selecting the most important subset from the expanded image of the previous iteration. The final similarity score incorporates both the initial and expanded images through a weighted sum (More details on the model design are in Section~\ref{sec: Method}). This process enables the model to capture both local and global patterns when comparing with a single prompt, thereby minimizing the biases introduced by similar features across different images. A key feature of LGCA is that, while it outperforms multiple baselines (see Section~\ref{sec: Exp}) by effectively capturing both local and global features of the image data, the time complexity remains comparable to that of the non-expanding model, increasing by at most a factor of 
$\log(\text{number of images} \cdot \text{number of captions})$ (see Section~\ref{sec: TimeAna}). In summary, our contributions are threefold.
\begin{enumerate}

    \item We propose LGCA, a framework designed to capture both local and global features of image data in the task of zero-shot image classification.
    \item We conduct experiments across multiple datasets and baselines to validate the performance of LGCA.
    \item We perform a theoretical analysis of the time complexity of LGCA and demonstrate that it maintains the same complexity as its initial non-expanding version.
\end{enumerate}

\section{Related Work} \label{sec: RW}

\subsection{Vision-Language Model}

Large-scale image-text pretraining has enabled vision-language models (VLMs) to learn robust representations for diverse tasks \cite{kim2021vilt,cho2021unifying,li2021align,xue2021probing}. CLIP \cite{radford2021learning}, trained on 400M image-text pairs, demonstrated strong zero-shot transfer and cross-modal generalization. Similarly, ALIGN \cite{jia2021scaling} showed that even noisy pretraining data can yield high-quality representations at scale. Building on this paradigm, models like FLAVA \cite{singh2022flava}, Florence \cite{yuan2021florence}, and BLIP \cite{li2022blip} advanced multimodal transformers and contrastive pretraining. More recent works like Kosmos2 \cite{peng2023kosmos}, LLaVA \cite{liu2023visual}, Qwen-VL \cite{bai2023versatile}, and Molmo \cite{deitke2024molmo}, further extend this direction through cross-attention for deeper fusion, generative pretraining for multimodal reasoning, and instruction tuning for better alignment with natural language queries.

\subsection{Prompting strategies for vision-language model}
\textbf{Text-Guided Prompting.} CLIP has proven effective for zero-shot tasks, but follow-up studies \cite{radford2021learning,zhou2022learning} show its performance is highly sensitive to prompt design, which often requires extensive manual tuning. To mitigate this, one research direction \cite{menon2022visual,pratt2023does} leverages LLMs like GPT-3 \cite{brown2020language} to generate class-specific descriptions that highlight discriminative features, thereby improving cross-modal alignment. Alternatively, WaffleCLIP \cite{roth2023waffling} bypasses LLMs by constructing prompts from random character n-grams, yet still achieves competitive results, revealing CLIP’s surprising robustness to nonsensical prompts. \\
\textbf{Image-Guided Prompting.} Visual prompting, the counterpart of textual prompting, steers model predictions by modifying the visual input itself. Lightweight methods like RedCircle \cite{shtedritski2023does} bias attention by simply encircling objects, though they require manual effort. More advanced approaches, such as FGVP \cite{yang2023fine}, reduce annotation needs by leveraging segmentation models \cite{kirillov2023segment} and Blur Reverse Masks to refine object boundaries, trading human supervision for system complexity. Another line of work explores image-cropping strategies that generate and rank candidate regions to highlight informative views \cite{nishi2009,han2022cropmix,thapa2024dragonfly}. Baseline methods often rely on random or multi-crop sampling \cite{liwca2024,zhongclipcrop2023}, combining many small with a few large crops to form a cheap but effective ensemble over viewpoint and scale.
\noindent \begin{figure}[t]
    \centering
    \includegraphics[width=1\linewidth]{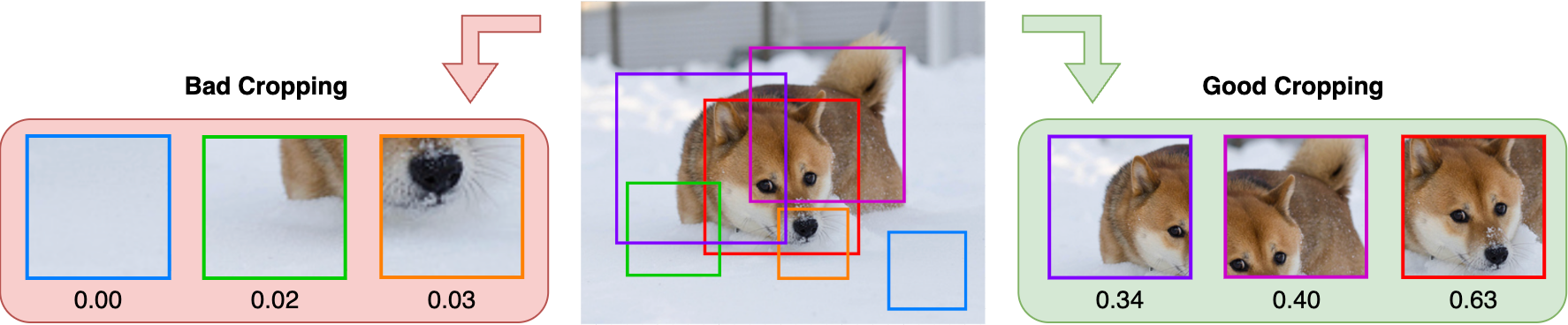}
    \caption{Local images are generated through cropping, with each crop weighted by its cosine similarity to the original image, indicating its level of correlation. In the illustration, the middle shows the original image, the left depicts low-correlation crops, and the right shows high-correlation crops.}
    \label{fig:img_crop} 
\end{figure}\\
\noindent\textbf{Test-time Prompt Tuning. }Test-time Prompt Tuning (TPT) \cite{shu2022tpt} adapts VLMs at inference by optimizing prompts with augmented views of test samples. It enables zero-shot generalization without labeled data and has shown strong performance in image classification \cite{feng2023diverse,abdul2023align,ma2023swapprompt,zhao2023test,zhang2024historical}. TPT refines prompts by enforcing consistency across a sample and its augmentations, but this requires multiple views and raises memory costs. WCA \cite{liwca2024} reduces this overhead by leveraging the inherent alignment ability of pretrained VLMs with labels' description prompts. Moreover, naive augmentations often yield overly simplistic variations, motivating our approach that uses cropping and progressive expansion during testing to better preserve semantics and avoid misleading small-scale features (see Figure \ref{fig:no_exp} and Section \ref{sec: Method}).

\section{Problem Formulation and Preliminaries} \label{sec: Prem}

\subsection{Problem Formulation} \label{Subsec: PS}

Let $\mathcal{I}$ denote the image domain and $\mathcal{L}$ the label domain, where labels are natural language tokens such as $\{\text{cat}, \text{dog}, \dots\}$. A pre-trained vision-language model consists of an image encoder $\phi: \mathcal{I} \to \mathbb{R}^k$ and a text encoder $\psi: \mathcal{L} \to \mathbb{R}^k$, mapping inputs into a joint $k$-dimensional embedding space. Here, $i \in \mathcal{I}$ represents an image and $c \in \mathcal{L}$ a candidate label.  
The zero-shot classification task then seeks to assign the most appropriate label $c$ to $i$, purely by comparing embeddings in this common space, while keeping the parameters of the pre-trained encoders frozen. In what follows, we outline approaches that are commonly employed to address this zero-shot classification problem.

\subsection{CLIP Zero-shot Transfer} \label{Subsec: Zero-shot Trans}

Following the work of \cite{radford2021learning}, the objective in zero-shot classification is to evaluate how well an image aligns with each candidate label by defining a scoring function $s: \mathcal{I} \times \mathcal{L} \to \mathbb{R}$. This function quantifies the semantic compatibility between an image $i \in \mathcal{I}$ and a label $l \in \mathcal{L}$. In practice, the score is obtained by comparing their encoder outputs using cosine similarity:
\begin{equation}
s(i, l \mid \phi, \psi) = \cos(\phi(i), \psi(l)).
\end{equation}
A larger value of $s(i, l)$ indicates stronger semantic correspondence between the image and the label. Consequently, classification reduces to selecting the label with the maximum score,
\[
l^{*} = \arg\max_{l \in \mathcal{L}} s(i, l),
\]
which assigns $i$ to the label whose textual representation is most closely aligned with its visual embedding.

\subsection{Enhancing Zero-shot Transfer Using Augmentation}
To further improve upon the method discussed in Subsection \ref{Subsec: Zero-shot Trans}, the works by \cite{menon2022visual,pratt2023does,liwca2024} introduced several approaches to enhance data representation through augmentation. These methods include\\
\textbf{Textual Augmentation.} For each category $l \in \mathcal{L}$, a large language model, denoted as $h(\cdot)$, can be employed to automatically generate multiple textual variants that elaborate on the defining attributes of the class. Instead of relying on a single label name, the LLM provides a diverse set of natural language descriptions that capture different perspectives of the same concept (e.g., visual features, contextual cues, or typical usage scenarios). Formally, the generated collection is written as
\begin{figure}[t]
    \centering
    \includegraphics[width=0.8\linewidth]{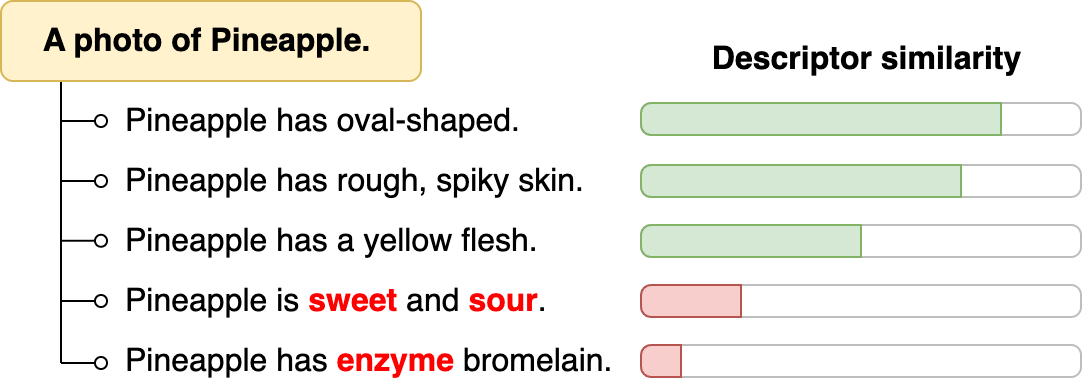}
    \caption{Visualization of similarity scores of text descriptions to the prompt “A photo of Pineapple.” Longer green lines indicate higher relevance, while shorter red lines mark low or incorrect matches. Descriptions that are irrelevant or incorrect are highlighted in red for clarity.}
    \label{fig:text_crop} 
\end{figure}
\begin{equation}
h(l) = \{\, l_j \,\}_{j=1}^{M},
\end{equation}
where $M$ is the number of synthesized descriptions and each $l_j$ corresponds to a semantically enriched prompt derived from the base label $l$. \\
\textbf{Visual Augmentation.} Random cropping is a widely used augmentation technique to improve the robustness of visual tasks \cite{liwca2024}. Given an image $i \in \mathcal{I}$ with width $w$ and height $h$, a crop size proportional to the smaller dimension is sampled, controlled by a parameter $\alpha \in (0,1)$. Formally, the operation is defined as
\begin{equation}
a(i, \alpha) = \{\, i_j \,\}_{j=1}^{N},
\end{equation}
where each $i_j$ is obtained by selecting a square subregion of side length
$\rho \cdot \min(h, w)$ with $\rho \sim \mathcal{U}(\alpha, 0.9)$, and resizing it to the original resolution. This generates $N$ variants of $i$, emphasizing different local regions or object parts.\\
\textbf{Weighted Aggregation.} To assess the relevance between original data and augmented data, image-to-image weights and text-to-text weights are introduced. Specifically, for image patches, we define a weight set 
\begin{equation}
\mathcal{W}_i = \{ w_{i_j} \}_{j=1}^{N},
\end{equation}
where $w_{i_j}$ reflects the significance of the $j$-th cropped variant of image $i$. Similarly, for textual descriptions, we assign weights
\begin{equation}
\mathcal{V}_c = \{ v_{c_j} \}_{j=1}^{M},
\end{equation}
where $v_{c_j}$ indicates the relevance of the $j$-th LLM-generated description $c_j$ for class $c$. Visualizations of applying weights to cropped images and to caption descriptions are shown in Figure \ref{fig:img_crop} and Figure \ref{fig:text_crop}, respectively. These weights not only address the uncertainty introduced by random cropping but also enable the model to emphasize the most informative visual and textual elements during cross-modal alignment. 
\section{Methodology} \label{sec: Method}

In this section, we formally introduce our proposed method, LGCA. The overall pipeline is illustrated in Figure \ref{fig:pipeline}. Specifically, we first describe a Cropping and Weight Assigning process for a specific image and caption, then we formalize the definition of an Expansion step, which serves as the foundation of our framework. We then describe how to combine these Expansion steps to obtain the overall model.
\subsection{Cropping and Weight Assignment}
\label{Subsec:CropWeight}
Recalling the definitions from Subsection~\ref{Subsec: PS}, we further introduce a cropping number $N$ and description number $M$, which specify the number of crops generated per image and the number of alternative descriptions generated per caption, respectively. For each image $i \in \mathcal{I}$, we construct a cropped image set $\mathcal{C}_i$ and an associated weight set $\mathcal{W}_i$. 
Specifically, we generate $N$ cropped versions of $i$ by applying a localized cropping function,  
\begin{equation}
\mathcal{C}_i \;=\; \left\{\, c_j = \phi\!\big(i, \gamma_j \min(H_i, W_i)\big) \;\middle|\; j = 1,\ldots,N \,\right\}, 
\end{equation}
where $H_i$ and $W_i$ denote the height and width of the image $i$, respectively, $\gamma_j$ is sampled from a uniform distribution $U(\alpha,\beta)$, and $\phi(\cdot)$ denotes the cropping operator. 
Each $c_j \in \mathcal{C}_i$ is assigned a weight relative to the original image $i$ by
\begin{equation}
w_j \;=\; \frac{\exp(s(c_j, i))}{\sum_{c \in \mathcal{C}_i} \exp(s(c, i))}, 
\end{equation}
where $s(\cdot,\cdot)$ is a similarity function. 
The resulting weights $\{w_j\}_{j=1}^N$ form the set $\mathcal{W}_i$. Similarly, for each caption $l \in \mathcal{L}$, we construct a set of alternative descriptions $\mathcal{D}_l$ and corresponding weights $\mathcal{V}_l$. 
We employ a large language model to produce $M$ descriptions of $l$, denoted $\{d_1, \ldots, d_M\}$, forming the set $\mathcal{D}_l$. 
Each description $d_j \in \mathcal{D}_l$ is  assigned a weight relative to the original caption $l$ by
\begin{equation}
v_j \;=\; \frac{\exp(s(d_j, l))}{\sum_{d \in \mathcal{D}_l} \exp(s(d, l))}.
\end{equation}
The resulting weights $\{v_j\}_{j=1}^M$ form the set $\mathcal{V}_l$.

\subsection{Expansion Step} 
In an expansion step, the model takes as input a set of cropped images $\mathcal{C}$, a set of descriptions $\mathcal{D}$, a set of image weights $\mathcal{W}$, a set of description weights $\mathcal{V}$, and a positive integer \texttt{topK}. 
First, each cropped image $c_s \in \mathcal{C}$ is passed through an image encoder to obtain an embedding vector $\hat{c}_s$ for $s = 1, \ldots, |\mathcal{C}|$, while each description $d_t \in \mathcal{D}$ is processed by a text encoder to produce an embedding $\hat{d}_t$ for $t = 1, \ldots, |\mathcal{D}|$. 
These embeddings are used to construct a cross-alignment matrix $\mathcal{A} \in \mathbb{R}^{|\mathcal{C}|\times|\mathcal{D}|}$ with entries
\[
\mathcal{A}_{s,t} \;=\; w_s \, v_t \, (\hat{c}_s^\top \hat{d}_t),
\]
where $w_s \in \mathcal{W}$ is the weight associated with image $c_s$ and $v_t \in \mathcal{V}$ is the weight associated with description $d_t$. 
The sum of all entries of $\mathcal{A}$ is set to the input scalar \texttt{score}
\[
\texttt{score} \;=\;  \sum_{s,t} \mathcal{A}_{s,t}.
\]
\begin{figure}[t]
    \centering
    \includegraphics[width=1\linewidth]{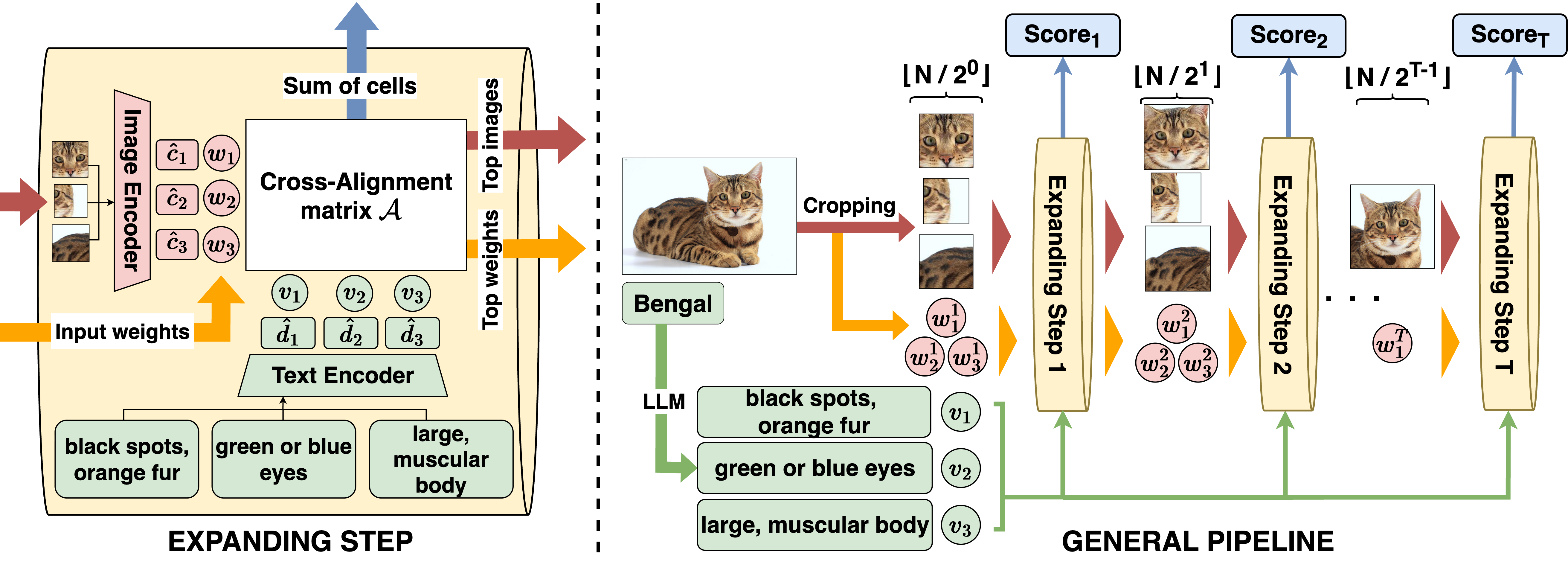}
    \caption{Left: Visualization of an Expansion Step. Right: Visualization of our general pipeline with $T$ Expansion Steps.}
    \label{fig:pipeline} 
\end{figure}
Next, the $\texttt{topK}$ largest entries of $\mathcal{A}$ are selected, and their indices are collected into a set $\mathcal{Z} \subseteq \{1,\dots,|\mathcal{C}|\}\times\{1,\dots,|\mathcal{D}|\}$. 
From this set, we derive the subset of images 
$
\hat{\mathcal{C}} \;=\; \{\, c_s \in \mathcal{C} \;|\; \exists\,t \;\text{such that}\; (s,t)\in \mathcal{Z} \,\}.
$ It should be noted that the cardinality of $\hat{\mathcal{C}}$ could be smaller than $K$ due to repetition. Then, each image in $\hat{\mathcal{C}}$ is spatially expanded within its original high-resolution frame. 
Concretely, for each $c_s \in \hat{\mathcal{C}}$, the image is expanded in both the horizontal and vertical directions by a fixed margin $\tau$, which serves to enlarge the cropped region within its parent image. 
The resulting expanded images are collected into the output set $\mathcal{C}_{\text{output}}$. Finally, for each image in $\mathcal{C}_{\text{output}}$, we re-calculate its weight with the initial image with the same formula stated in Subsection \ref{Subsec:CropWeight} to get the new image weight set $\hat{\mathcal{W}}$. The Expansion Step outputs the expanded image set $\mathcal{C}_{\text{output}}$, the score \texttt{score}, and the new image weights $\hat{\mathcal{W}}$, which are then propagated to the next stage of the LGCA pipeline. Visualization of the Expansion Step are shown in Figure~\ref{fig:pipeline}-Left, while examples are given in Figure~\ref{fig: dataset}.
\subsection{Overall Structure of LGCA}

We now introduce the overall structure of our proposed method, LGCA.  
Let $\mathcal{I}$ denote the set of images and $\mathcal{L}$ the set of captions.  
Given a cropping number $N$ and an expansion rate $\tau$, LGCA computes the similarity between an image $i \in \mathcal{I}$ and a caption $l \in \mathcal{L}$ as follows. First, we apply the cropping procedure with cropping number $N$ to generate:
$
\mathcal{C}_i, \;\mathcal{W}_i, \;\mathcal{D}_l, \;\mathcal{V}_l,
$
where $\mathcal{C}_i$ denotes the cropped image regions of $i$, $\mathcal{W}_i$ their associated weights, and $\mathcal{D}_l, \mathcal{V}_l$ are the cropped caption segments and their embeddings, respectively.  

Next, we determine the number of iterations $T$ by choosing the largest integer $T$ such that
$
\left\lfloor \frac{N}{2^T} \right\rfloor = 1.
$
We perform $T$ Expansion Steps. At iteration $j \in \{1, \dots, T\}$, LGCA takes as input the cropped image set $\mathcal{C}_i^{(j-1)}$, and the image weights $\mathcal{W}_i^{(j-1)}$ from the previous step, together with the caption descriptions and weights $(\mathcal{D}_l, \mathcal{V}_l)$. At this step, LGCA choose the positive integer $\texttt{topK}$ to be $ \left\lfloor \frac{N}{2^{j}} \right\rfloor$. The Expansion Step then outputs $
\mathcal{C}_i^{(j)}, \;\;\mathcal{W}_i^{(j)}, \;\;\texttt{score}^{(j)}
$, which are used for the next iteration. After finishing all $T$ steps, the similarity between image $i$ and caption $l$ is computed as a weighted sum of the intermediate scores:
\[
\text{Sim}(i,l) = \sum_{j=1}^{T} \alpha_j \cdot \texttt{score}^{(j)},
\]
where $\texttt{score}^{(j)}$ is the score at step $j$ and the weights $\{\alpha_j\}_{j=1}^T$ are hyperparameters fine-tuned for each experiment. We denote this similarity between image $i$ and caption $l$ by LGCA$(i,l)$. \\
\textbf{Image-caption matching procedure.} To complete the zero-shot image classification task, we assign each image 
$i \in \mathcal{I}$ to the label that maximizes its similarity under LGCA.  
Formally, for each $i \in \mathcal{I}$, the predicted label is given by
\[
l_i \;=\; \arg\max_{l \in \mathcal{L}} \; \text{LGCA}(i,l),
\]
\section{Time Complexity of the Expansion Step} \label{sec: TimeAna}
Recalling the definitions from Subsection \ref{Subsec: PS}, we define a non-expanding model $\mathbf{Q}$. For each image-caption pair $(i, l)$, $\mathbf{Q}$ applies only the Cropping and Weight Assignment steps, encodes the cropped images and descriptions, computes weights, and forms the Cross-Alignment matrix. The overall similarity score is given by the sum of all entries in this matrix, followed by an image-caption matching procedure analogous to LGCA. Thus, $\mathbf{Q}$ can be seen as a variant of LGCA without the Expansion Step, as in \cite{li2024visual}. The following theorem establishes the complexity relationship between LGCA$_{\mathbf{Q}}$ and $\mathbf{Q}$.
\begin{theorem}\label{theo: main theorem}
     Consider a non-expanding model $\mathbf{Q}$. Let $I, L \in \mathbb{R}_{>0}$ be positive real numbers such that the time complexity of $\mathbf{Q}(\mathcal{I}, \mathcal{L})$ is given by $\mathcal{O}(H\times N^I \times M^L)$, for any image and caption datasets   $\mathcal{I}$ and $\mathcal{L}$ where $N$ and $M$ denote the number of crops per image and descriptions per caption, respectively. $H$ is the complexity of the image-caption matching procedure and depends only on the cardinality of $\mathcal{I}$ and $\mathcal{L}$. Then, the time complexity of $\text{LGCA}_{\mathbf{Q}}(\mathcal{I}, \mathcal{L})$ is $\mathcal{O}(H\times N^I\times M^L+HNM(\log M +\log N))$.
\end{theorem}
\begin{proof}
    Assume that $\text{LGAC}_{\mathbf{Q}}(\mathcal{I}, \mathcal{L})$ has $T$ Expansion Steps. Consider a pair of image and caption $(x,y)$ and let $\mathcal{C}_x$, $\mathcal{D}_y$ be the set of cropped images and descriptions, respectively. Denote $\mathcal{C}_x^{(i)}$ the image set used at step $i$ for $i\in\{1,\dots,T\}$. Hence, following the definition of LGAC, we have $\mathcal{C}_x^{(1)}=\mathcal{C}_x$ and $|\mathcal{C}_x^{(i)}|\leq \Big\lfloor \dfrac{N}{2^{i-1}} \Big\rfloor \quad \forall i\in\{2,\dots,T\}$. Furthermore, for all $j\in \{1,\dots,T\}$, at Expansion Step $j$ of LGAC we essentially do 3 things:\\
    First, run $\mathbf{Q}(\mathcal{C}_x^{(j)})$ without the final image-caption matching procedure. The complexity of this procedure is at most  $\mathcal{O}\left(\Big\lfloor \dfrac{N}{2^{j-1}} \Big\rfloor^I \times M^L\right)$. Then, sort to find the top $\Big\lfloor \dfrac{N}{2^{j}} \Big\rfloor$ largest cosine similarity out of $|\mathcal{C}_x^{(1)}|\times M$ cross-alignment similarity score. The complexity of this procedure is $\mathcal{O}\left(|\mathcal{C}_x^{(1)}|M\log(|\mathcal{C}_x^{(1)}| M)\right)$, which is at most $\mathcal{O}\left(\Big\lfloor \dfrac{N}{2^{j}} \Big\rfloor M\log\left(\Big\lfloor \dfrac{N}{2^{j}} \Big\rfloor M\right)\right)$. Lastly, expand at most $\Big\lfloor \dfrac{N}{2^{j}} \Big\rfloor$ images with the largest similarity score. The complexity of this procedure is at most $\mathcal{O}\left(\Big\lfloor \dfrac{N}{2^{j}} \Big\rfloor\right)$. Hence, the total time complexity at  Expansion Step $j$ is 
    \begin{align*}
        \mathcal{O}\left(\Big\lfloor \dfrac{N}{2^{j-1}} \Big\rfloor^I \times M^L + \Big\lfloor \dfrac{N}{2^{j}} \Big\rfloor M\log\left(\Big\lfloor \dfrac{N}{2^{j}} \Big\rfloor M\right)\right)
    \end{align*}
    By combining for all $j\in \{1,\dots, T\}$ and adding the image-caption matching procedure, we yield the complexity of LGCA$_{\mathbf{Q}}$ as follows
    \begin{align}\label{eq: comp 1}
    \mathcal{O}\left(H\times\displaystyle\sum_{j=0}^{T-1}\Big\lfloor \dfrac{N}{2^{j}} \Big\rfloor^I \times M^L + H\times\displaystyle\sum_{j=1}^{T}\Big\lfloor \dfrac{N}{2^{j}} \Big\rfloor M\log\left(\Big\lfloor \dfrac{N}{2^{j}} \Big\rfloor M\right)\right)
    \end{align}
    Notice that
    \begin{align*}
        &\displaystyle\sum_{j=0}^{T-1}\Big\lfloor \dfrac{N}{2^{j}} \Big\rfloor^I \times M^L + \displaystyle\sum_{j=1}^{T}\Big\lfloor \dfrac{N}{2^{j}} \Big\rfloor M\log\left(\Big\lfloor \dfrac{N}{2^{j}} \Big\rfloor M\right)\\
        &\leq N^IM^L\left(\displaystyle\sum_{j=0}^{T-1} \dfrac{1 }{2^{Ij}}\right)+NM\left(\log N\left(\displaystyle\sum_{j=1}^{T} \dfrac{1 }{2^{j}}\right)+(\log M-\log(2))\left(\displaystyle\sum_{j=1}^{T} \dfrac{1 }{2^{j}}\right)\right)\\
        &\leq C(N^IM^L+NM(\log M +\log N))
    \end{align*}
    The last inequality is true by applying the inequality$\displaystyle\sum_{j=s}^t\dfrac{1}{2^j}<\dfrac{1}{2^{s-1}}\forall s,t\in\mathbb{Z}_{>0}$\\Thus, by combining with the result in \ref{eq: comp 1}, we conclude that the time complexity of LGAC$_{\mathbf{Q}}$ is $\mathcal{O}(HN^IM^L+HNM(\log M +\log N))$.
\end{proof}
In the literature, non-expanding baseline models typically have $I, L \ge 1$ due to the construction of the Cross-Alignment Matrix. Hence, Theorem~\ref{theo: main theorem} shows that for a non-expanding $\mathbf{Q}$, adding the expansion step increases the time complexity by at most a factor of $\log N+ \log M$, and in most cases remains essentially unchanged when $I, L > 1$. This demonstrates that in LGCA, although many Expansion steps are added, which significantly improve performance, the impact on time complexity is minimal or negligible. 
\noindent \begin{figure}[t]
    \centering
    \includegraphics[width=0.8\linewidth]{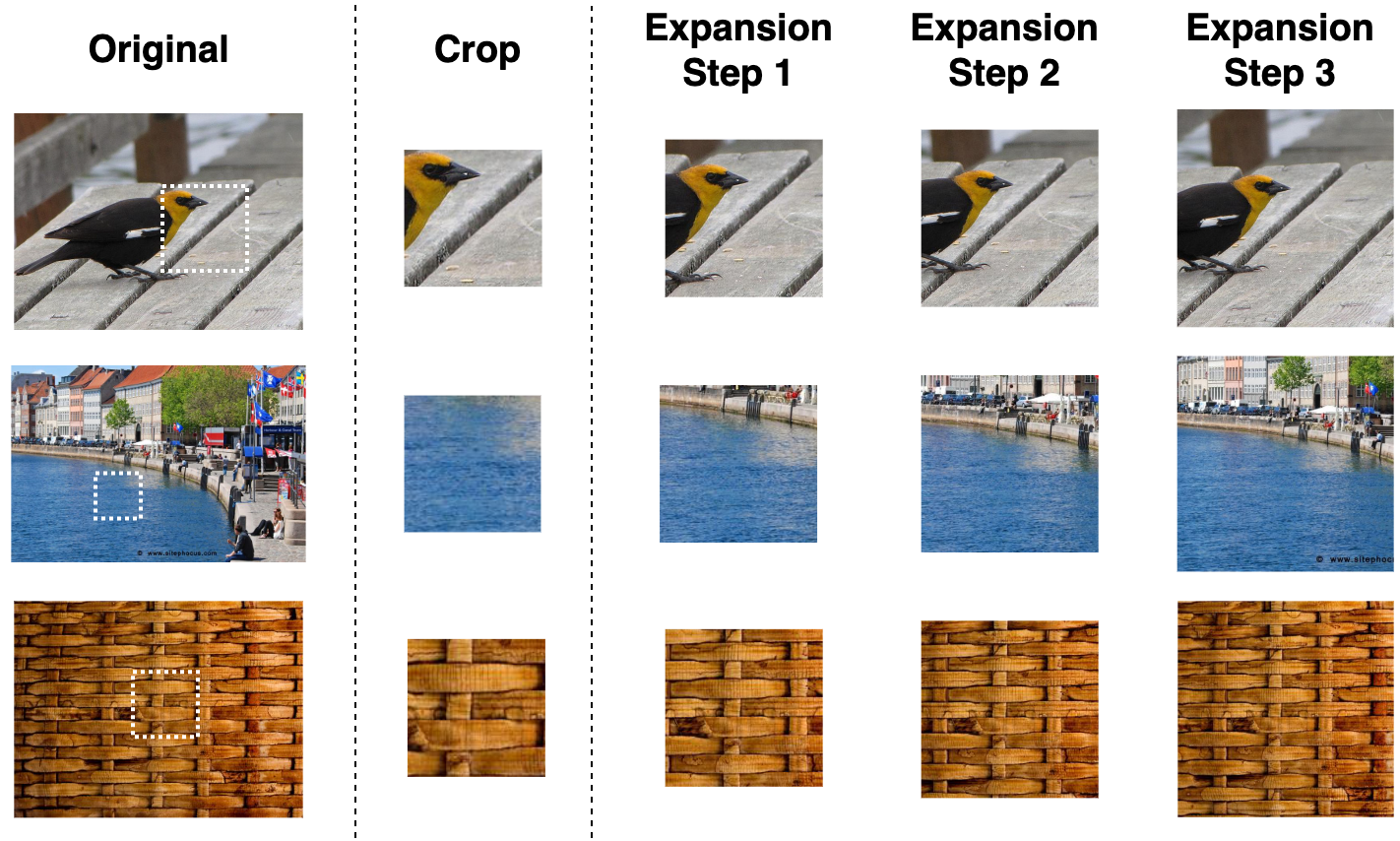}
    \caption{\textbf{Visualization of images through Expansion steps}. In each row, we use an image from CUB\_200\_2011, Place365, and the DTD dataset, respectively. The leftmost column shows the original image. The second column presents the cropped region, and the subsequent columns illustrate the progressively expanded regions.}
    \label{fig: dataset} 
\end{figure}
\section{Experiments} \label{sec: Exp}
\textbf{Datasets.} We test our method on five benchmark datasets for zero-shot classification: Oxford-IIIT Pets dataset \cite{Parkhi2012CatsAD} featuring common dog and cat species; CUB\_200\_2011 dataset \cite{WahCUB_200_2011} for fine-grained bird classification; DTD dataset \cite{Cimpoi_2014_CVPR} containing diverse in-the-wild textures; Food101 dataset \cite{bossard14} of food images; and Place365 dataset \cite{Zhou2018PlacesA1} designed for large-scale scene recognition.\\
\textbf{Baselines. } In the context of zero-shot image classification, we consider the baselines outlined in \cite{liwca2024}: CLIP \cite{radford2021learning}, which utilizes a simple template, \texttt{"A photo of \{class\}"}; CLIP-E, an ensemble variant of CLIP that customizes the prompt text for each task, for example on Oxford-IIIT Pets, \texttt{"a photo of a \{\}, a type of pet."}; CLIP-D \cite{menon2022visual}, which generates descriptions with the help of LLMs; CupL \cite{pratt2023does}, producing LLM-based descriptions of higher quality than CLIP-D; and Waffle \cite{roth2023waffling}, which replaces LLM-generated descriptions with randomly generated characters and words. Among these, the prompts for CLIP-D and CupL are sourced from the authors’ public repositories \cite{menon2022visual,pratt2023does}, while the remaining baselines use hand-crafted or code-generated prompts provided by \cite{liwca2024}.\\
\textbf{Parameters and Fine-tuning.}  
For the Cropping and Weight Assignment step, we employ the \texttt{RandomCrop} strategy, where the crop size is sampled uniformly from the range \((\alpha, \beta)\). Therefore, our method is controlled by two main parameters: the cropping ratio bounds \((\alpha, \beta)\) and the number of crops \(N\) generated per image. Following \cite{liwca2024}, the upper bound is fixed at \(\beta = 0.9\). The lower bound \(\alpha\) is dataset-specific: we set \(\alpha = 0.7\) for Place365, where capturing larger regions better reflects scene-level information, and \(\alpha = 0.5\) for all other datasets, where smaller crops help emphasize fine-grained object details. To increase regional diversity, each image is augmented with \(N = 100\) crops. For the Expansion step, our method introduces two additional hyperparameters. The first is the initial \texttt{topK}, which determines the number of highest-scoring crop–description pairs before expansion; we set \texttt{topK} = 10 to balance diversity with reliability. The second is the expansion margin \(\tau\), which is the scaling factor applied during expansion. We consider two values, \(\tau \in \{1.1, 1.25\}\), and select the one that aligns best with the dataset characteristics.\\
\textbf{Implementation details. } All experiments are carried out on a system with an 8-core CPU and 32 GB RAM, relying on the default multi-core setup to parallelize crop generation and evaluation. Each benchmark dataset is tested with two widely used CLIP backbones: ViT-B/32 and ViT-B/16. 
\begin{table*}[t]
\centering
\renewcommand*{\arraystretch}{1.0}
\caption{Zero-shot image classification accuracy (\%) of LGCA and baseline methods across datasets using two CLIP backbones (ViT-B/32 and ViT-B/16). Bold and underlined values denote the best and second-best results, respectively. $\Delta$ indicates the performance gain of LGCA over the strongest baseline.
}
\label{temp}
\resizebox{\textwidth}{!}{%
\begin{tabular}{cccccccccccccccccccccccc}
\toprule

\multirow{2}{*}{\textbf{Method}}
& \multicolumn{2}{c}{Oxford-IIIT Pets} 
& \multicolumn{2}{c}{CUB\_200\_2011} 
& \multicolumn{2}{c}{DTD} 
& \multicolumn{2}{c}{Food-101} 
& \multicolumn{2}{c}{Place365}\\
\cmidrule(lr){2-3} \cmidrule(lr){4-5} \cmidrule(lr){6-7} \cmidrule(lr){8-9} \cmidrule(lr){10-11} 
              & B/32 & B/16 & B/32 & B/16 & B/32 & B/16 & B/32 & B/16 & B/32 & B/16 \\
\midrule
CLIP     & 85.06 & 88.20 & 51.33 & 55.95 & 43.30 & 43.24 & 82.31 & 88.23 & 38.60    & 39.55 \\
CLIP-E   & 87.44 & 89.07 & 52.81 & 56.32 & 44.36 & 44.73 & 84.01 & 88.73 & 39.27    & 40.24 \\
CLIP-D   & 84.49 & 87.63 & 52.69 & 56.99 & 44.04 & 46.38 & 84.11 & 88.78 & 38.69   & 39.68 \\
Waffle   & 85.36 & 86.48 & 52.11 & 57.01 & 42.55 & 44.41 & 83.91 & 89.06 & 39.63    & 40.74 \\
CupL     & 87.38 & 91.69 & 49.67 & 54.26 & 47.55 & 47.82 & 84.08 & 88.87 & 38.83    & 39.93 \\
\midrule
WCA      & \underline{89.08} & \underline{92.05} & \underline{56.72} & \underline{59.63} & \underline{48.06} & \underline{50.53} & \underline{86.02} & \underline{89.83} &  \underline{40.26}   & \underline{40.95} \\
\midrule
Ours     & \textbf{90.13} & \textbf{92.97} & \textbf{57.47} & \textbf{61.27} & \textbf{48.25} & \textbf{50.74} & \textbf{86.35} & \textbf{89.95} & \textbf{40.52} & \textbf{41.08} \\
\midrule
$\Delta$ 
& \pos{+1.05} & \pos{+0.92} & \pos{+0.75} & \pos{+1.64} 
& \pos{+0.19} 
& \pos{+0.21} & \pos{+0.33} & \pos{+0.12} 
& \pos{+0.26} & \pos{+0.13} \\
\bottomrule
\end{tabular}
}
\label{tab: res}
\end{table*}
\subsection{Results}
In these experiments, we use classification accuracy as the evaluation metric, which measures the proportion of correctly predicted samples over the total number of samples. The results on standard zero-shot benchmarks are summarized in Table~\ref{tab: res}. Our method consistently outperforms all baselines across the evaluated datasets. The most substantial improvement is observed on CUB-200-2011, where we achieve a 1.64\% gain with ViT-B/16. On the Oxford-IIIT Pets dataset, our approach provides roughly +1\% improvement under both ViT configurations. On more challenging datasets such as DTD and Place365, which involve repetitive patterns and complex scenes (see Figure~\ref{fig: dataset}), our approach consistently achieves performance gains. These results highlight its robustness and adaptability across diverse and demanding dataset characteristics.

\section{Conclusion} \label{sec: Con}
In this work, we introduce LGCA, a framework for zero-shot image classification that first extracts local features and then iteratively selects and expands the most salient regions. This enables the model to capture both localized and global representations, avoiding confusion from small-scale similarities across distinct images. We demonstrate that LGCA maintains constant computational complexity even with multiple expansion steps, highlighting its efficiency. For future work, one can explore how this approach can be adapted to other modalities or how to generalize the Expansion Step to work for both image and caption.
\newpage

\section*{Acknowledgements}
This research is supported by research funding from Faculty of Information Technology, University of Science, Vietnam National University - Ho Chi Minh City. We would like to thank AISIA Lab for supporting us during this project.
%
%
%
\bibliographystyle{splncs04}
\bibliography{samplepaper}

\end{document}